\documentclass[runningheads]{llncs}
\usepackage[T1]{fontenc}

\usepackage{hyperref}

\usepackage{graphicx,verbatim}

\usepackage{comment, todonotes,amssymb, amsmath,cleveref, url, multirow, booktabs,subcaption}

\usepackage{color}

\usepackage{orcidlink}

\usepackage[misc]{ifsym}

\Crefname{section}{sec.}{secs.}
\Crefname{section}{Sec.}{Secs.}
\Crefname{theorem}{thm.}{thms.}
\Crefname{theorem}{Thm.}{Thms.}
\Crefname{figure}{fig.}{figs.}
\Crefname{figure}{Fig.}{Figs.}
\crefname{equation}{eq.}{eqs.}
\Crefname{equation}{Eq.}{Eqs.}

\begin{document}
\title{Tied Prototype Model for \\Few-Shot Medical Image Segmentation}

\author{Hyeongji Kim\inst{1}\textsuperscript{(\Letter)}\orcidlink{0000-0003-3935-6656} 
\and
Stine Hansen\inst{2}\orcidlink{0000-0002-0962-6160} \and
Michael Kampffmeyer\inst{1,3}\orcidlink{0000-0002-7699-0405}}
\authorrunning{H. Kim et al.}
\institute{Department of Physics and Technology, UiT - The Arctic University of Norway, Tromsø, Norway \\
\href{mailto:hyeongji.kim@uit.no}{hyeongji.kim@uit.no}\and 
Norwegian Centre for Clinical Artificial Intelligence (SPKI), University Hospital of North-Norway, Tromsø, Norway \and Norwegian Computing Center, Oslo, Norway }
    
\maketitle              %
\begin{abstract}
Common prototype-based medical image few-shot segmentation (FSS) methods %
model foreground and background classes using class-specific prototypes. However, given the high variability of the background, a more promising direction is to focus solely on foreground modeling, treating the background as an anomaly---an approach introduced by ADNet. %
Yet, ADNet %
faces three key limitations: %
dependence on a single prototype per class, a focus on binary classification, and fixed thresholds that fail to adapt to patient and organ variability. To address these shortcomings, 
we propose the Tied Prototype Model (TPM), a principled reformulation of ADNet with tied prototype locations for foreground and background distributions. Building on %
its probabilistic foundation, TPM naturally extends to multiple prototypes and multi-class segmentation while effectively separating %
non-typical background features. %
Notably, both extensions lead to improved segmentation accuracy. %
Finally, we leverage naturally occurring class priors to define an ideal target for adaptive thresholds, boosting %
segmentation performance. %
Taken together, TPM %
provides a fresh perspective on prototype-based FSS for medical image segmentation. The code can be found at %
\url{https://github.com/hjk92g/TPM-FSS}. %


\end{abstract}

\section{Introduction}

Medical image segmentation is a critical component of numerous clinical applications such as diagnostics \cite{tsochatzidis2021integrating} and treatment planning \cite{chen2021deep}. While supervised deep learning approaches can achieve good performance, %
their applications are constrained by the limited availability of annotated medical images. To address this, few-shot segmentation (FSS) approaches have been proposed %
to effectively adapt models trained with labeled datasets to new, previously unseen, classes.

The dominant FSS approach employs prototype networks \cite{snell2017prototypical}, first introduced by PANet \cite{wang2019panet}, 
where each class is represented by a single prototype obtained through mean average pooling (MAP). 
Recognizing %
the limitation of a single prototype in modeling rich feature variants, PPNet \cite{liu2020part} proposed using multiple prototypes for each class to improve expressiveness. %
More recently, there has been growing interest in leveraging self-supervised learning (SSL) to circumvent the need for labeled medical data altogether. %
ALPNet \cite{ouyang2020self} 
 pioneered the use of SSL %
by using pseudo-labels generated from superpixels, with additional local prototypes aiding intra-class local information. %
However, the background class typically exhibits significant %
heterogeneity, %
making it difficult to model using a fixed number of prototypes. %
Recognizing this, ADNet \cite{hansen2022anomaly} employs %
an anomaly-detection-inspired approach, %
focusing on modeling the foreground prototype %
while treating the background as anomalous. %
It assigns anomaly scores %
to each spatial location, using a fixed threshold at inference for segmentation. %
This simple method has demonstrated greater robustness to background heterogeneity relative to ALPNet and has been successfully used for %
the detection of brain tumors \cite{balasundaram2023foreground}, ischemic stroke lesions \cite{tomasetti2023self}, and lung lesions \cite{tian2025multilevel}. %

However, ADNet model \cite{hansen2022anomaly} has three main drawbacks. 
First, its reliance on %
a single foreground prototype limits the expressiveness of the foreground class in medical images with significant intra-class variation. %
While non-anomaly-detection multi-prototype %
methods exist~\cite{ouyang2020self,cheng2024few}, they struggle to capture the full diversity of the background class. %
Meanwhile, %
existing anomaly-detection-based multi-prototype variants %
\cite{salahuddin2022self,zhu2023few,zhao2024cpnet,zhu2024partition} of ADNet involve 
cumbersome modifications such as the introduction of %
new layers or hyperparameters. Hence, it is desirable for a more grounded yet simple approach that efficiently leverages multiple foreground prototypes at the feature level. 
Second, ADNet is limited to %
binary classification settings. %
While ADNet++ \cite{hansen2023adnet++} extends it to multi-class classification, it only adapts the inference phase and thus does not effectively model the class relationships in training. %
Finally, ADNet's fixed threshold is insufficient to accommodate %
the inherent variability across patients and organs %
in diverse applications. While %
adaptive thresholds have been proposed to address this by utilizing support features \cite{cheng2024frequency}, query information \cite{shen2023q}, or both \cite{zhu2024learning}, their %
threshold values %
are %
learned through cross-entropy-based losses, which %
may not necessarily optimize segmentation accuracy as further discussed in \Cref{sec:ICP}. %

\begin{figure}[tp]
     \centering
     \begin{subfigure}[b]{0.1925\textwidth}%
         \centering
         \includegraphics[width=\textwidth]{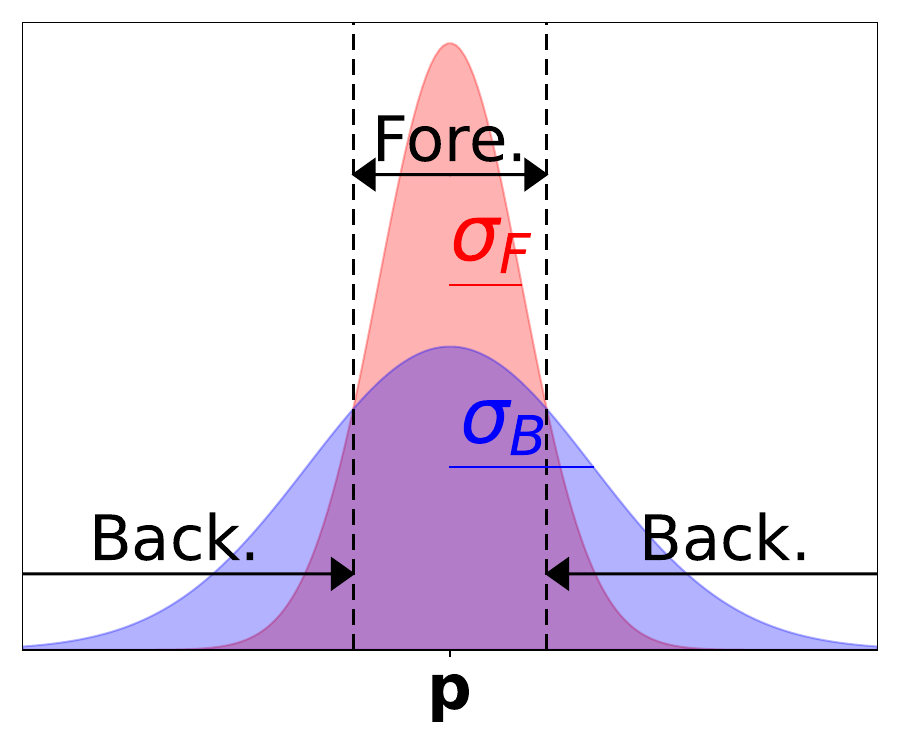}
         \caption{1D-distrib.}
         \label{fig:tied_1d}
     \end{subfigure}
     \begin{subfigure}[b]{0.1925\textwidth}%
         \centering
         \includegraphics[width=\textwidth]{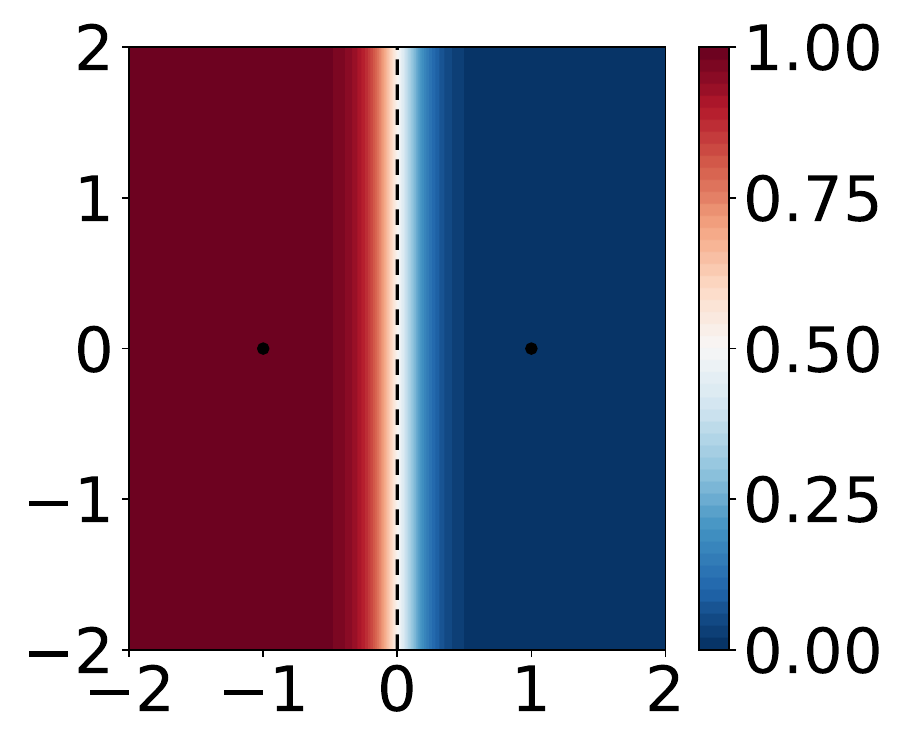}
         \caption{Std. model}
         \label{fig:std_model}
     \end{subfigure}
     \begin{subfigure}[b]{0.1925\textwidth}%
         \centering
         \includegraphics[width=\textwidth]{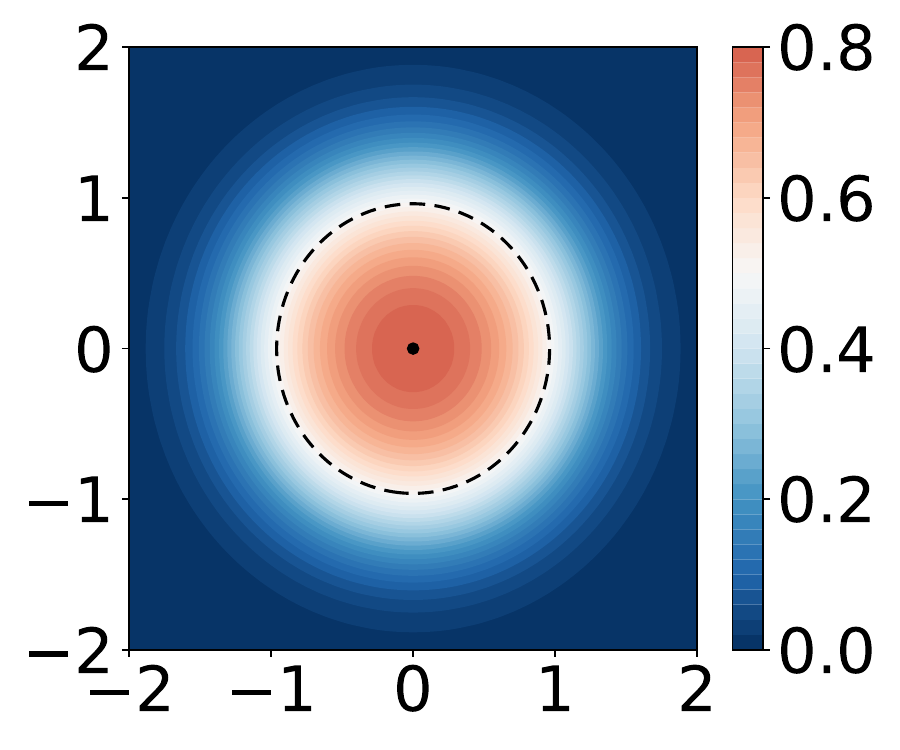}
         \caption{TPM-SP}
         \label{fig:tied_model}
     \end{subfigure}
     \begin{subfigure}[b]{0.1925\textwidth}%
         \centering
         \includegraphics[width=\textwidth]{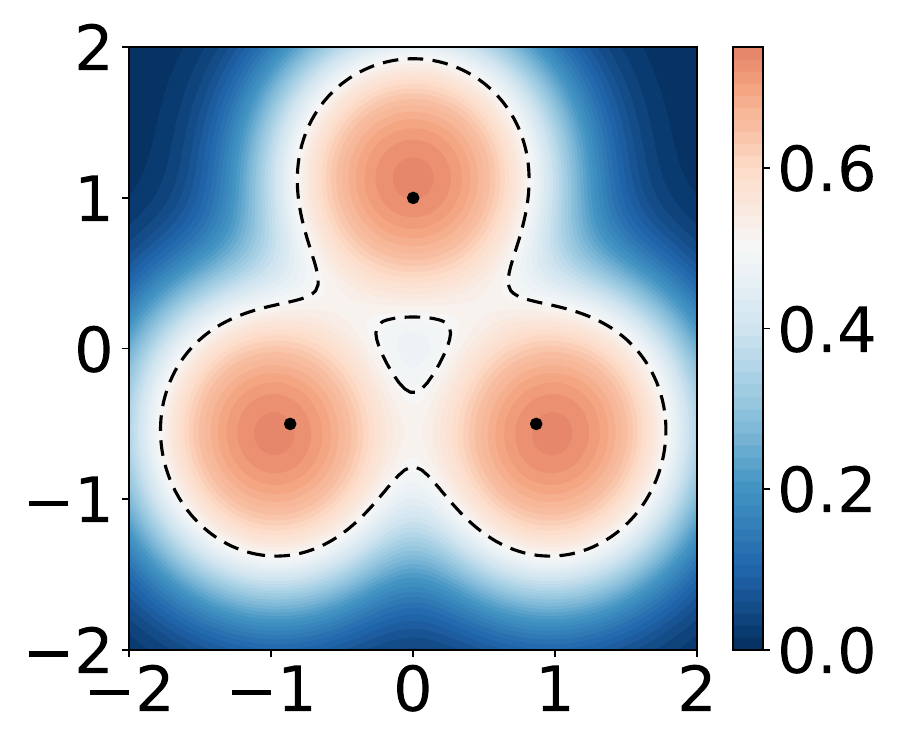}
         \caption{TPM-MP}
         \label{fig:tied_model_MP}
     \end{subfigure}
     \begin{subfigure}[b]{0.1925\textwidth}%
         \centering
         \includegraphics[width=\textwidth]{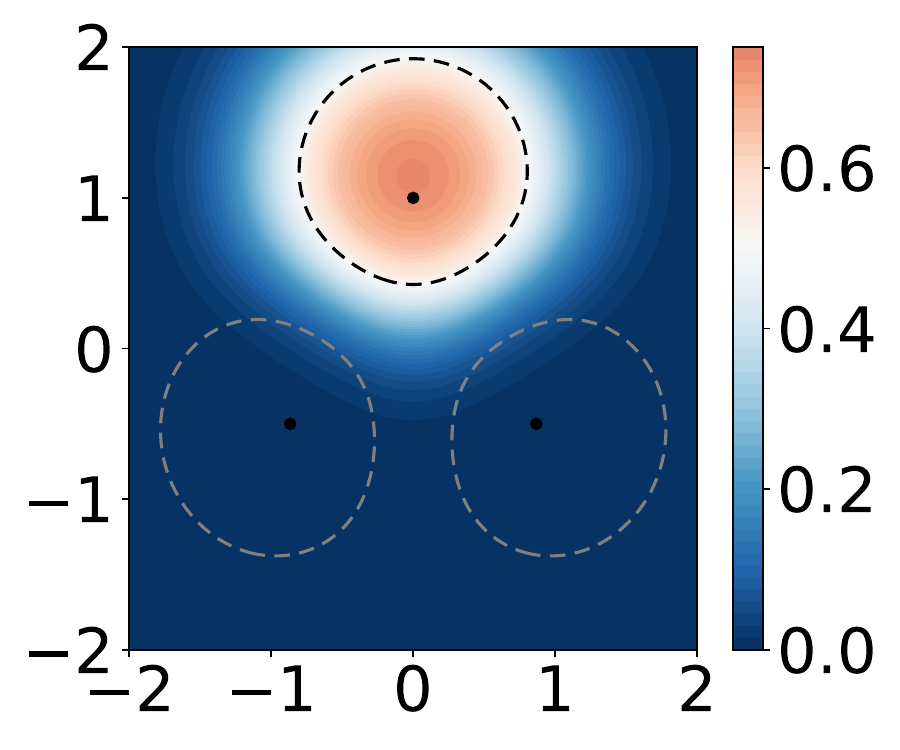}
         \caption{TPM-MC}
         \label{fig:tied_model_MC}
     \end{subfigure}
        \caption{%
        (a) TPM class distributions. The (first) foreground class probability \(p(F|x)\), 
        based on (b) the standard %
        prototype model, (c) %
        TPM %
        with a single prototype (SP), similar to the ADNet~\cite{hansen2022anomaly}---which is limited to spherical space, %
        (d) TPM with multiple prototypes (MP), and (e) TPM for multi-class (MC) classification. %
        Dashed lines indicate decision boundaries for the shown (black) and the other (\textcolor{gray}{gray}) foreground classes.}
        \label{fig:TP_model}
\end{figure}

To address these issues, we take a step back and propose the Tied Prototype Model (TPM), a %
principled reformulation of ADNet \cite{hansen2022anomaly} that leverages tied (shared) prototype locations for foreground and background distributions. %
As visualized in \Cref{fig:TP_model}, the distributional foundation of TPM allows for its natural extension %
to both multiple prototypes and multi-class classification settings while retaining its simplicity and ability to separate non-typical background features. %
In particular, we propose a %
Gaussian mixture model (GMM)-based approach for multi-prototype %
segmentation, %
demonstrating its effectiveness in improving segmentation accuracy. 
Moreover, we introduce multi-foreground training for multi-class segmentation and show that it consistently offers superior representation learning capability. %
Lastly, by employing naturally occurring class priors %
and feature distance distributions, we %
provide %
an ideal %
threshold estimation approach, %
aiming to %
enhance segmentation accuracy.

\section{Preliminaries}

\subsection{Self-Supervised Few-Shot Learning Problem Setting} %

We adopt the %
self-supervised few-shot learning from \cite{hansen2022anomaly}, %
where pseudo-labels for unlabeled volumes are generated by performing supervoxel segmentation and randomly sampling a supervoxel as the foreground. %
Two 2D slices with the supervoxel %
are then selected to serve as support and query images, with one being augmented. %
The support image \(\mathbf{X}^{s}\) and its pseudo-label \(\mathbf{Y}^{s}\), where \(\mathbf{Y}^{s}_c\) denotes the mask of class \(c\), guide the segmentation of %
the query image \(\mathbf{X}^{q}\), all of size (\(H, W\)). %
A feature extractor network \(f_{\theta}\) encodes images into feature maps %
\(\mathbf{F}^s=f_{\theta}(\mathbf{X}^s)\) and \(\mathbf{F}^q=f_{\theta}(\mathbf{X}^q)\), with spatial dimensions (\(H', W'\)) %
and feature dimension \(d\). %
The upsampled predicted query label \(\hat{\mathbf{Y}}^{q}\) is then compared to 
label \(\mathbf{Y}^{q}\), training the network \(f_{\theta}\) with a cross-entropy-based segmentation loss. %

\subsection{Anomaly Detection-Inspired Few-Shot Segmentation} %

In ADNet \cite{hansen2022anomaly}, the foreground prototype \(\mathbf{p}\) %
is computed by upsampling the feature maps to the original size %
and applying %
masked average pooling (MAP) as: %
\begin{align}\label{eq:proto}
    \mathbf{p} = \frac{\sum_{\mathbf{r}} \mathbf{F}^s (\mathbf{r}) \odot \mathbf{Y}^s_F (\mathbf{r})}{\sum_{\mathbf{r}} \mathbf{Y}^s_F (\mathbf{r})},
\end{align}
with \(\odot\) as the Hadamard product, \(F\) the foreground class, 
and \(\mathbf{r}\) the pixel location. %
It defines the anomaly score %
\(S(\mathbf{r})=-\alpha \cos(\mathbf{F}^q (\mathbf{r}), \mathbf{p})\) for each query feature \(\mathbf{F}^q (\mathbf{r})\), %
where \(\alpha=20\) is a scaling factor following \cite{oreshkin2018tadam}, and \(\cos(\cdot,\cdot)\) is the cosine similarity. %
Then, %
ADNet estimates foreground class probability as:
\begin{align}\label{eq:ADNet}
    p(F|\mathbf{F}^q (\mathbf{r}))=1-\text{sig}(S(\mathbf{r})-T_S),
\end{align}
where \(\text{sig}(\cdot)\) is the sigmoid function with a steepness parameter \(\kappa=0.5\), and $T_S$ is a learnable anomaly score threshold. %
This estimate is then upsampled and, given the binary task focus, is combined with its complement to generate the predicted label \(\hat{\mathbf{Y}}^{q}\).

\section{Method}

\subsection{Tied Prototype Model}\label{sec:tied_proto_SP} %

The tied prototype model (TPM), with our established equivalence to ADNet \cite{hansen2022anomaly} under a certain condition, adopts a probabilistic view of discriminative classification. This provides valuable insights for advancing few-shot medical image segmentation while addressing ADNet's limitations without modifying its architecture. %

As illustrated in \Cref{fig:tied_1d}, the key idea of our model is to use class distributions with a tied (shared) center position \(\mathbf{p}\) for %
foreground and background classes while differing in their dispersion parameters. 
This contrasts with standard classification, %
which uses distinct centers for different classes. 
Though %
counterintuitive, enforcing the same center for both foreground and background classes %
is essential in separating the background class, %
as shown in \Cref{fig:tied_model} and \Cref{thm:eqv}. %

TPM assumes that the foreground and background class distributions %
follow multivariate normal distributions, \(\mathcal{N}(\mathbf{p},\sigma_F^2I)\) and \(\mathcal{N}(\mathbf{p},\sigma_B^2I)\), respectively. Here, the tied prototype \(\mathbf{p}\) serves as the location parameter, while \(\sigma_F\) and \(\sigma_B\) are class standard deviations, with \(\sigma_F<\sigma_B\). %
The symbol \(I\) denotes %
the identity matrix. %
By applying 
Bayes' theorem, %
the foreground class probability is given by: 
\begin{align}\label{eq:tied_proto}%
    p(F|\mathbf{F}^q (\mathbf{r}))=\frac{p_F \phi(\mathbf{p},\sigma_F;\mathbf{r})}{p_F \phi(\mathbf{p},\sigma_F;\mathbf{r})+p_B \phi(\mathbf{p},\sigma_B;\mathbf{r})}%
        =\frac{1}{1+\frac{p_B}{p_F}\frac{ \phi(\mathbf{p},\sigma_B;\mathbf{r})}{\phi(\mathbf{p},\sigma_F;\mathbf{r})}}
\end{align}
where \(p_F\) and \(p_B\) denote %
class priors, %
and \(\phi(\mathbf{p},\sigma;\mathbf{r})\) represents %
the normal distribution density function %
with a mean \(\mathbf{p}\) and covariance matrix \(\sigma^2I\). %

\Cref{fig:tied_model} visualizes the estimated class probability of TPM %
with a single prototype, %
assigning %
higher probabilities to features %
closer to the prototype \(\mathbf{p}\) while giving lower values to those further from the prototype. %
This behavior highlights the model's effectiveness in distinguishing typical features from diverse non-typical features, which is particularly useful for foreground and background classification.
As illustrated in \Cref{fig:std_model}, this beneficial property is absent in the standard prototype model, which relies on distinct (non-tied) prototype positions and assigns class probabilities along the direction of the prototypes' difference.

\begin{theorem}\label{thm:eqv}
In one foreground classification %
with a single prototype with %
a unit spherical embedding \(\mathbb{S}^{d-1}\), %
the tied prototype model is equivalent to %
ADNet~\cite{hansen2022anomaly}. %
\end{theorem}

\begin{proof}
To establish %
equivalence, we must %
show that \Cref{eq:ADNet,eq:tied_proto} are identical under %
a 
proper parameter correspondence. %
We first define \(\text{TP}(\mathbf{r})\) to simplify \Cref{eq:tied_proto} and replace the multiplication with the exponential function, yielding: %
\begin{align}
    \text{TP}(\mathbf{r})=%
    \frac{1}{1+\frac{p_B}{p_F}\frac{ \phi(\mathbf{p},\sigma_B;\mathbf{r})}{\phi(\mathbf{p},\sigma_F;\mathbf{r})}}=\frac{1}{1+\exp \left (\ln \left ( \frac{ \phi(\mathbf{p},\sigma_B;\mathbf{r})}{\phi(\mathbf{p},\sigma_F;\mathbf{r})} \right )  +\ln \left ( \frac{p_B}{p_F} \right ) \right )}.
\end{align}
By substituting the %
density function \(\phi(\cdot,\cdot)\) with its definition, we obtain: 
\begin{align}
    \text{TP}(\mathbf{r})
    =\frac{1}{1+\exp \left ( \frac{1}{2} \left\|\mathbf{F}^q (\mathbf{r})-\mathbf{p} \right\|^2 \left ( \frac{1}{\sigma_F^2}-\frac{1}{\sigma_B^2} \right )+d\ln \left ( \frac{\sigma_F}{\sigma_B} \right )  +\ln \left ( \frac{p_B}{p_F} \right ) \right )}.
\end{align}
We then set \(\alpha=2 \left ( \frac{1}{\sigma_F^2}-\frac{1}{\sigma_B^2} \right )\) and \(T_S=2\ln{\left (  \frac{p_F}{p_B}\right )}-2d\ln{\left (  \frac{\sigma_F}{\sigma_B}\right )}-\alpha\) to obtain:
\begin{align}
    \text{TP}(\mathbf{r})
    =\frac{1}{1+\exp \left (\frac{1}{2}  \left ( \left ( \frac{1}{2}\left\|\mathbf{F}^q (\mathbf{r})-\mathbf{p} \right\|^2 -1\right )\alpha -T_S \right )\right )}.
\end{align}
Since we use the unit spherical embedding \(\mathbb{S}^{d-1}\), embedding vectors \(\mathbf{F}^q (\mathbf{r})\) and \(\mathbf{p}\) satisfy the equation \(\frac{2-\left\|\mathbf{F}^q (\mathbf{r})-\mathbf{p} \right\|^2}{2} =\cos(\mathbf{F}^q (\mathbf{r}),\mathbf{p})\), leading to: %
\begin{align}\label{eq:TP_ADNet}
    \text{TP}(\mathbf{r})=\frac{1}{1+\exp \left (\frac{1}{2}  \left ( -\alpha \cos(\mathbf{F}^q (\mathbf{r}), \mathbf{p}) -T_S\right )\right )}.
\end{align}
Using the definitions of the anomaly score \(S(\mathbf{r})\) and the sigmoid function \(\text{sig}(\cdot)\), we can observe %
that \Cref{eq:TP_ADNet} is equivalent to \Cref{eq:ADNet}. \hfill \( \square\)%
\end{proof}

\Cref{thm:eqv} reveals %
that our TPM, under specific conditions, %
is a reparameterized %
ADNet \cite{hansen2022anomaly}. Unlike ADNet which inherently %
uses spherical embedding due to its reliance on cosine similarity, our model can utilize %
other geometries. %
Furthermore, it enables extensions beyond a single prototype as discussed below.

\subsection{Binary %
Classification with Multiple Prototypes}\label{sec:tied_proto_MP}

Rooted in its distributional formulation, TPM 
naturally extends to multiple %
prototypes. As illustrated in \Cref{fig:tied_model_MP}, this extension allows it to capture diversity in foreground features while retaining the ability to distinguish non-typical background features. %

From the foreground features, we apply the EM algorithm \cite{yang2020prototype} which alternates between the E-step and M-step, to extract multiple prototypes \(\mathbf{p_{\mathbf{m}}}\) along with their weights \(w_{\mathbf{m}}\), satisfying \(\sum w_{\mathbf{m}}=1\). When multiple prototypes %
are obtained for the foreground class, we can simply use %
the same prototypes %
for the background class as well. %
Specifically, the Gaussian mixture model (GMM) %
\(\sum_{\mathbf{m}} w_{\mathbf{m}} \phi(\mathbf{p_{\mathbf{m}}},\sigma_F;\mathbf{r})\) is used for the foreground class distribution, %
while %
the corresponding %
GMM \(\sum_{\mathbf{m}} w_{\mathbf{m}} \phi(\mathbf{p_m},\sigma_B;\mathbf{r})\) represents %
the background class distribution. %
Using Bayes' theorem, %
the foreground class probability \(p_{\text{MP}}(F|\mathbf{F}^q (\mathbf{r}))\), estimated from multiple prototypes and associated class priors, %
is given by: %
\begin{align}\label{eq:tied_proto_MP}%
    p_{\text{MP}}(F|\mathbf{F}^q (\mathbf{r}))=\frac{p_{F;\text{MP}} \sum\limits_{\mathbf{m}} w_{\mathbf{m}} \phi(\mathbf{p_m},\sigma_F;\mathbf{r})}{p_{F;\text{MP}} \sum\limits_{\mathbf{m}} w_{\mathbf{m}}\phi(\mathbf{p_m},\sigma_F;\mathbf{r})+p_{B;\text{MP}} \sum\limits_{\mathbf{m}} w_{\mathbf{m}}\phi(\mathbf{p_m},\sigma_B;\mathbf{r})}.
\end{align}

\subsection{Multi-Class Classification}%

The distributional foundation of TPM %
allows for another expansion: multi-class classification. %
While multiple prototypes can be used for each foreground class, this work focuses on the simplified case of using a single prototype \(\mathbf{p}_{F_i}\), obtained by MAP in \Cref{eq:proto}, for foreground class \(F_i\). %
\Cref{fig:tied_model_MC} visualizes the class probability \(p(F_1|x)\) of the first foreground class with prototype \(\mathbf{p}_{F_1}=(0,1)\). The model effectively distinguishes %
features from other foreground classes while ensuring the separation of non-typical background features. %

With a similar assumption on class distributions as in \Cref{sec:tied_proto_SP,sec:tied_proto_MP}, the class probability \(p_{\text{MC}}(F_i|\mathbf{F}^q (\mathbf{r}))\) of foreground class \(F_i\) can be obtained as:
\begin{align}\label{eq:tied_proto_MC_F}%
    p_{\text{MC}}(F_i|\mathbf{F}^q (\mathbf{r}))=\frac{p_{F_i}  \phi(\mathbf{p}_{F_i},\sigma_F;\mathbf{r})}{\sum_{i'} p_{F_{i'}}\phi(\mathbf{p}_{F_{i'}},\sigma_F;\mathbf{r})+\sum_{i'} p_B \phi(\mathbf{p}_{F_{i'}},\sigma_B;\mathbf{r})},
\end{align}
where \(p_{F_i}\) is the class prior of %
\(F_i\). The background class probability can then be computed as \(p_{\text{MC}}(B|\mathbf{F}^q (\mathbf{r}))=1-\sum_{i'} p_{\text{MC}}(F_i'|\mathbf{F}^q (\mathbf{r}))\).

Ignoring the background class, i.e., when \(p_B=0\), %
\Cref{eq:tied_proto_MC_F} corresponds to the normalized softmax \cite{wang2017normface}. %
This suggests that the model enables training to separate different foreground features into their respective prototype positions while %
pushing background features away from the prototypes. 
This is not possible in multi-class classification with ADNet++ \cite{hansen2023adnet++}, as its use of the max operation over foreground classes inherently makes it %
a piecewise binary classification model. %

\subsection{Targeting the Ideal %
Threshold}\label{sec:ICP}

\begin{figure}[tp]
     \centering
     \begin{subfigure}[b]{0.225\textwidth}%
         \centering
         \includegraphics[width=\textwidth]{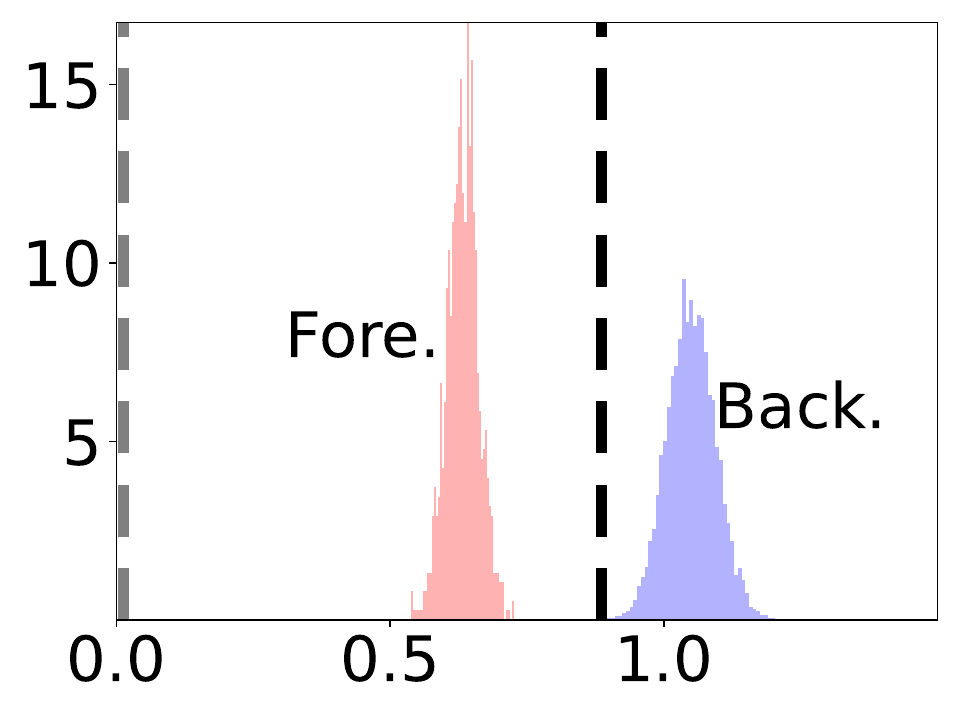}
         \caption{Dist. distrib.}
         \label{fig:dist_distrib}
     \end{subfigure}
     \begin{subfigure}[b]{0.225\textwidth}%
         \centering
         \includegraphics[width=\textwidth]{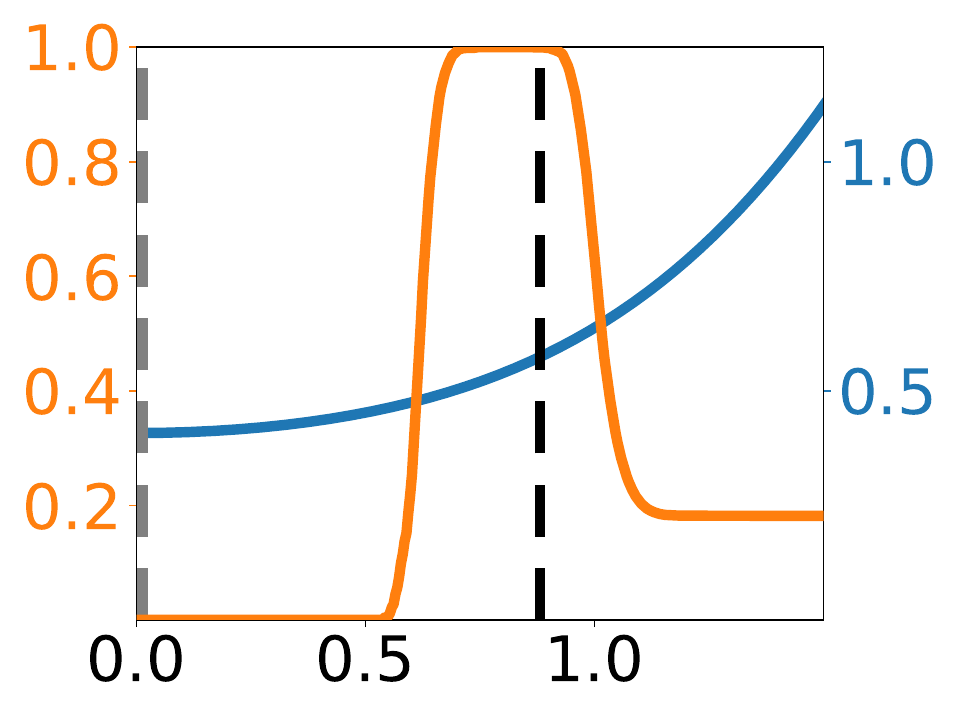}
         \caption{Varying %
         \(T_D\)}
         \label{fig:T_D_CE_dice}
     \end{subfigure}
     \begin{subfigure}[b]{0.225\textwidth}%
         \centering
         \includegraphics[width=\textwidth]{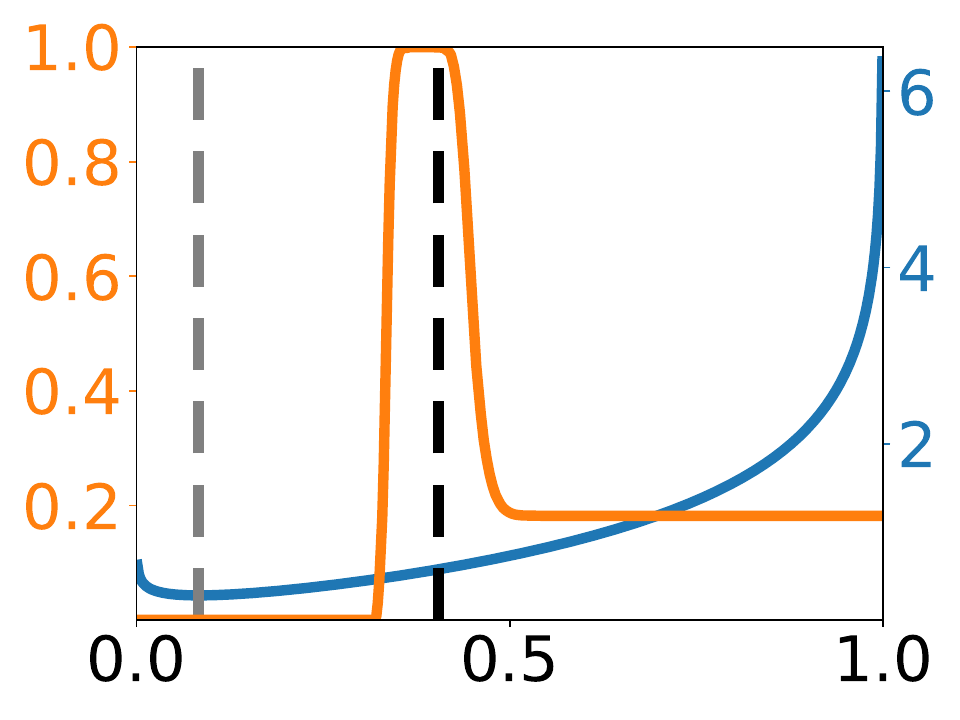}
         \caption{Varying %
         \(p_F\)}
         \label{fig:p_F_CE_dice}
     \end{subfigure}
        \caption{%
        (a) Distributions of distances to a prototype for %
        Gaussian %
        samples projected onto a sphere. Dice score (\textcolor{orange}{orange}) vs. %
        CE (\textcolor{blue}{blue}) for (b) varying distance threshold \(T_D\) and (c) varying foreground class prior \(p_F\). %
        Dashed lines indicate %
        values for 
        the ideal threshold (black) and the CE-minimizing threshold (\textcolor{gray}{gray}).}
        \label{fig:threshold_ex}
\end{figure}

The foreground-background imbalance is a critical challenge in few-shot medical segmentation. While adaptive thresholds \cite{shen2023q,cheng2024frequency,zhu2024learning} attempt %
to tackle this with cross-entropy (CE), as illustrated in \Cref{fig:threshold_ex}, %
minimizing CE does not necessarily yield the highest Dice score, showing %
that CE minimization does not directly optimize segmentation accuracy. Furthermore, using the \emph{ideal} %
proportion---where the predicted foreground pixel counts match the true counts---%
achieves a near-optimal Dice score. Inspired by this and the significant performance gains observed in \cite{boudiaf2021few}, we propose ideal %
thresholds of training query images as %
targets for learning adaptive thresholds. %
For brevity, 
we consider the single-prototype binary classification setting, which readily extends %
to the other settings by taking the minimum distances across prototypes or classes. 

First, we compute %
the upsampled Euclidean distance between features and prototypes. 
We then calculate the ideal %
distance threshold (IDT) \(T_D^*\) for each training query image. %
Mathematically, \(T_D^*=\frac{D\text{;sort}_{|F|}+D\text{;sort}_{|F|+1}}{2},\) %
where \(\text{D;sort}\) is a vector of distances sorted in ascending order %
and \(|F|\) is the foreground pixel count. By leveraging %
sorted distances, IDT ensures the number of predicted foreground pixels matches the true count \(|F|\). %
To employ \Cref{eq:tied_proto}, we convert IDT into the corresponding ideal class prior (ICP) as %
\(p_F^* =1-\text{sig}(-T_D^{*2} \left (  \frac{1}{\sigma_F^2} - \frac{1}{\sigma_B^2}\right )-2d\ln(\frac{\sigma_F}{\sigma_B}))\).
Replacing %
\(p_F\) and the background prior with ICP \(p_F^*\), in \Cref{eq:tied_proto}, %
shifts %
its decision boundary distance to IDT \(T_D^*\). %

\section{Experiments}

\subsection{Setup} %

\textbf{Dataset.} %
Two public abdominal datasets are used for evaluation: ABD-MRI and ABD-CT. %
ABD-MRI \cite{kavur2021chaos} %
consists of 20 abdominal MRI scans with annotations for the liver, left kidney (L.K.), right kidney (R.K.), and spleen (Spl.). 
ABD-CT \cite{landman2015miccai} %
contains 30 abdominal CT scans with the same organ annotations as ABD-MRI. For supervoxel generation, based on %
\cite{hansen2022anomaly,hansen2023adnet++}, %
 we set the size parameter to 5000 for ABD-MRI and 2000 for ABD-CT. %

\textbf{Implementation Details.}
As in  \cite{hansen2022anomaly}, we conduct %
experiments using %
5-fold cross-validation, performing three training runs per split, %
and leverage a single support image slice to segment the entire query volume and measure the mean Dice score. %
We initialize the feature extractor \(f_{\theta}\) with a pre-trained ResNet-101 \cite{he2016deep}, using randomly transformed and reshaped 256×256 images, and optimize for 50k iterations using SGD with a learning rate of \(10^{-3}\). %
Features are 256-dimensional and normalized onto a spherical embedding, 
with \(d=1\) %
for simplicity. %
We run experiments in 
PyTorch (v1.9.0) on an NVIDIA RTX 3090 GPU. %

\textbf{Training and Threshold Options.}
To showcase the effectiveness %
of our approaches, we compare results across various options. For binary segmentation, we %
assess single-prototype-trained models using both single- and five-prototype evaluations. %
The baseline is 
standard ADNet \cite{hansen2022anomaly}, trained with the additional %
T loss, %
which regularizes \(T_S\). %
In contrast, %
our TPM training %
excludes the T loss and only optimizes the segmentation loss. %
To assess the impact of aligning predicted and true pixel counts, we report TPM results using the CE-trained threshold (CE-T), estimated ideal class prior (ICP) values from training data---specifically, the average estimation (AvgEst) and a linear estimation (LinEst), with the latter %
based on the support foreground size \(|F|\) and query slice location---as well as the oracle ideal class prior (OCP) %
for test queries. %
For multi-class segmentation, we compare results from single- and five-foreground training, %
using ADNet++ \cite{hansen2023adnet++} as the 
baseline. %
While it is possible to consider each class count separately, we report LinEst estimated by matching the sum of all foreground class counts. %

\begin{table}[tp]
\footnotesize
\caption{%
Binary segmentation. %
Best result in \textbf{bold}, second-best \underline{underlined}. %
* denotes our contributions and \textcolor{gray}{OCP} is the oracle result %
using 
test labels. %
}\label{tab:binary-EP2-both}
\begin{tabular}{|l|l|l|l|l|l|l|l|l|l|l|l|}
\hline
\multirow{2}{*}{Proto.}&\multirow{2}{*}{Inference} & \multicolumn{5}{|c|}{ABD-MRI \cite{kavur2021chaos}}& \multicolumn{5}{|c|}{ABD-CT \cite{landman2015miccai}}\\ \cline{3-12} %
{}&{} & Liver & L.K. & R.K.& Spl.& Mean& Liver & L.K. & R.K.& Spl.& Mean\\
\hline
\multirow{5}{*}{Single}&ADNet\cite{hansen2022anomaly} %
& \( \underline{74.86}\)& \( 59.41\) & \(84.27\)& \( 59.54\)& \(  69.52\) &\( 56.18 \)& \( 34.82 \) & \(37.56  \)& \( 49.14 \)& \( 44.42   \)\\ %
{}&CE-T*& \( \bf{75.31}\)& \( 56.91\) & \( 78.87\)& \( 56.93\)& \(  67.00\)&\( \underline{71.76 }\)& \( 29.27 \) &\( 28.99 \)& \(  46.90 \)& \(  44.23  \)\\
{} & AvgEst*%
& \( 69.94 \)& \( \underline{70.38 }\) & \( \bf{85.99 }\)& \( 69.40 \)& \( 73.93  \)&\( 33.36  \)& \( 39.42 \) & \( \bf{41.49 }\)& \( 56.47\)& \(  42.68  \)\\ 
{} & LinEst* %
& \( 73.98 \)& \( \bf{71.23 }\) & \( 85.64 \)& \( 69.11 \)& \(  \underline{74.99 }\)&\( 65.34  \)& \( 38.92 \)& \( 40.88 \)& \( 58.49\)& \(  \underline{50.91  }\)\\
{} & \textcolor{gray}{OCP %
}& \textcolor{gray}{\( 85.71 \)}& \textcolor{gray}{\( 87.19 \)} & \textcolor{gray}{\( 89.47 \)}& \textcolor{gray}{\( 81.05 \)}& \textcolor{gray}{\(  85.85 \)}& \textcolor{gray}{\( 81.01\)}& \textcolor{gray}{\( 65.65 \)}& \textcolor{gray}{\( 67.34 \)}& \textcolor{gray}{\(  76.15 \)}& \textcolor{gray}{\(  72.54  \)}\\ \hline
\multirow{4}{*}{\shortstack{Multi \\ (5)}}&CE-T*& \( 74.85\)& \( 54.92 \) & \( 77.60 \)& \( 56.75 \)& \( 66.03 \)&\( \bf{75.50 }\)& \( 28.48 \) &\( 28.79 \)& \( 46.06 \)& \( 44.71 \)\\
{} & AvgEst* %
& \( 73.60 \)& \( 68.08 \) & \( \underline{85.82 }\)& \( \underline{69.65 }\)& \( 74.29 \)&\( 46.26 \)& \( \underline{39.97 }\) &\( 40.49 \)& \( \underline{58.71 }\)& \( 46.36 \)\\
{} & LinEst* %
& \( 74.83 \)& \( 70.18 \) & \( 85.61 \)& \( \bf{70.18 }\)& \( \bf{75.20 }\)&\( 70.40 \)& \( \bf{40.08 }\) & \( \underline{41.40 }\)& \( \bf{58.96 }\)& \(\bf{ 52.71 }\)\\ %
{} & \textcolor{gray}{OCP %
}& \textcolor{gray}{\( 85.39 \)}& \textcolor{gray}{\( 86.00 \)} & \textcolor{gray}{\( 89.00 \)}& \textcolor{gray}{\( 81.04 \)}& \textcolor{gray}{\( 85.36 \)}&\textcolor{gray}{\( 82.96 \)}& \textcolor{gray}{\( 64.14 \)}& \textcolor{gray}{\( 66.00 \)}& \textcolor{gray}{\( 76.14 \)}& \textcolor{gray}{\( 72.31 \)}\\
\hline
\end{tabular}
\end{table}

\subsection{Results}
\Cref{tab:binary-EP2-both} presents the binary segmentation results, confirming the merits of the ICP and multi-prototype (MP) approaches. In the single-prototype (SP) setting, while ADNet \cite{hansen2022anomaly} tends to outperform CE-T (ADNet equivalent without T loss), %
AvgEst generally yields higher Dice scores than ADNet, except for the liver class, which has a relatively large foreground size. Moreover, LinEst further improves performance by incorporating size information, revealing ICP's %
benefit in this setting. The notably higher Dice scores of the OCP, which leverages oracle information, suggest significant room %
for improvement with advanced ICP estimation. Compared to SP predictions, %
the proposed MP extension further enhances the performance of the estimated ICPs.
\Cref{tab:multi-both} presents the results for TPM's multi-class segmentation extension, demonstrating the efficacy of ICP and multi-foreground (MF) training. Focusing on single foreground (SF) training, TPM with estimated ICPs significantly outperforms the ADNet++ baseline \cite{hansen2023adnet++}. Furthermore, MF consistently enhances TPM’s performance over SF. %
This highlights the superior representation learning capability of the proposed extension of our TPM. %

\begin{table}[tp]
\footnotesize
\caption{%
Multi-class segmentation. * denotes our contributions. %
}\label{tab:multi-both}
\begin{tabular}{|l|l|l|l|l|l|l|l|l|l|l|l|}
\hline
\multirow{2}{*}{Train}& \multirow{2}{*}{Inference} & \multicolumn{5}{|c|}{ABD-MRI \cite{kavur2021chaos}}& \multicolumn{5}{|c|}{ABD-CT \cite{landman2015miccai}}\\ \cline{3-12}
{}&{} & Liver & L.K. & R.K.& Spl.& Mean& Liver & L.K. & R.K.& Spl.& Mean\\
\hline
\multirow{2}{*}{SF} & ADNet++\cite{hansen2023adnet++}& \(64.13\)& \(54.62\)&  \(59.86\)& \(41.36\)& \( 54.99\)& \underline{\(65.58\)}& \(20.67\)& \(16.66\)& \(19.59\)& \(30.63\)\\
{} & LinEst* %
& \underline{\(65.50\)} & \underline{\(69.43\)} &  \underline{\(70.72\)}& \underline{\(47.15\)} & \underline{\(  63.20\)}& \(61.97\)& \underline{\(31.77\)}& \underline{\(24.46\)}& \(\bf{29.15}\)& \underline{\(36.84\)}\\ \hline %
\multirow{2}{*}{\shortstack{MF\\(5)}} & ADNet++\cite{hansen2023adnet++} & \(57.22\)& \(45.85\)& \(54.57\)& \(29.93\)& \( 46.89\)&\(61.14\)& \(15.43\)& \(12.00\)& \(12.88\)& \(25.36\)\\
{} & LinEst* %
& \(\bf{71.91}\) & \(\bf{71.00}\)& \(\bf{73.85}\)& \(\bf{49.94}\) & \( \bf{66.67}\)&\(\bf{66.84}\)& \(\bf{37.90}\)& \(\bf{27.55}\)& \underline{\(28.95\)}& \(\bf{40.31}\)\\ 
\hline
\end{tabular}
\end{table}

\section{Conclusion}

In this work, we introduce the tied prototype model (TPM), a probabilistic reformulation of %
ADNet \cite{hansen2022anomaly} that advances its %
capabilities.
Notably, TPM enables seamless extensions to multi-prototype and multi-class segmentation. %
Furthermore, 
we highlight %
the significance of ideal thresholds as target thresholds. 
Our experimental results demonstrate the performance improvements %
of each component, %
paving %
new research directions 
in prototype-based few-shot segmentation. %

\begin{credits}
\subsubsection{\ackname} 
This work was supported by The Research Council of Norway (Visual Intelligence, grant no. 309439 as well as FRIPRO grant no. 315029 and IKTPLUSS grant no. 303514).

\subsubsection{\discintname}
The authors have no competing interests to declare that are relevant to the content of this article.
\end{credits}

\bibliographystyle{splncs04}
\bibliography{My_library}

\end{document}